\documentclass{article}
\usepackage[utf8]{inputenc}
\usepackage{amssymb,amsmath,amsthm}
\usepackage{enumerate}
\usepackage[margin=1.25in]{geometry}
\newtheorem{theorem}{Theorem}
\newtheorem{lemma}{Lemma}

\usepackage{soul}
\usepackage{url}
\usepackage[hidelinks]{hyperref}
\usepackage[utf8]{inputenc}
\usepackage[small]{caption}
\usepackage{graphicx}
\usepackage{amsmath}
\usepackage{booktabs}
\usepackage{algorithm}
\usepackage{algorithmic}
\usepackage{bbm}
\usepackage{tabularx} 
\usepackage{amsmath, amsfonts}  
\usepackage{graphicx} 
\usepackage{cite} 
\usepackage{tcolorbox}
\usepackage{amsmath}
\usepackage{amsthm}
\usepackage{multirow}
\usepackage{amssymb}  
\usepackage{tabularx} 
\usepackage{amsmath, amsfonts}  
\usepackage{mathtools}

\DeclareMathOperator{\sign}{sign}

\newtheorem{definition}{Definition}

\newtheorem{fact}{Fact}
\DeclareMathOperator*{\argmax}{arg\,max}

\title{Thompson Sampling on Symmetric $\alpha$-Stable Bandits}

\author{
Abhimanyu Dubey and Alex Pentland\\
Massachusetts Institute of Technology\\
\texttt{\{dubeya, pentland\}@mit.edu}
}
\date{}
\begin{document}

\maketitle

\begin{abstract}
Thompson Sampling provides an efficient technique to introduce prior knowledge in the multi-armed bandit problem, along with providing remarkable empirical performance. In this paper, we revisit the Thompson Sampling algorithm under rewards drawn from $\alpha$-stable distributions, which are a class of heavy-tailed probability distributions utilized in finance and economics, in problems such as modeling stock prices and human behavior. We present an efficient framework for $\alpha$-stable posterior inference, which leads to two algorithms for Thompson Sampling in this setting. We prove finite-time regret bounds for both algorithms, and demonstrate through a series of experiments the stronger performance of Thompson Sampling in this setting. With our results, we provide an exposition of $\alpha$-stable distributions in sequential decision-making, and enable sequential Bayesian inference in applications from diverse fields in finance and complex systems that operate on heavy-tailed features.
\end{abstract}

\section{Introduction}
The multi-armed bandit (MAB) problem is a fundamental model in understanding the \textit{exploration-exploitation} dilemma in sequential decision-making. The problem and several of its variants have been studied extensively over the years, and a number of algorithms have been proposed that optimally solve the bandit problem when the reward distributions are well-behaved, i.e. have a finite support, or are sub-exponential. 

The most prominently studied class of algorithms are the Upper Confidence Bound (UCB) algorithms, that employ an ``optimism in the face of uncertainty'' heuristic~\cite{auer2002finite}, which have been shown to be optimal (in terms of regret) in several cases~\cite{cappe2013kullback, bubeck2013bandits}. Over the past few years, however, there has been a resurgence in interest in the Thompson Sampling (TS) algorithm~\cite{thompson1933likelihood}, that approaches the problem from a Bayesian perspective.

Rigorous empirical evidence in favor of TS demonstrated by~\cite{chapelle2011empirical} sparked new interest in the theoretical analysis of the algorithm, and the seminal work of~\cite{agrawal2012analysis, agrawal2013further, russo2014learning} demonstrated the optimality of TS when rewards are bounded in $[0, 1]$ or are Gaussian. These results were extended in the work of~\cite{korda2013thompson} to more general, exponential family reward distributions. The empirical studies, along with theoretical guarantees, have established TS as a powerful algorithm for the MAB problem.

However, when designing decision-making algorithms for complex systems, we see that interactions in such systems often lead to heavy-tailed and power law distributions, such as modeling stock prices~\cite{bradley2003financial}, preferential attachment in social networks~\cite{mahanti2013tale}, and online behavior on websites~\cite{kumar2010characterization}.

Specifically, we consider a family of extremely heavy-tailed reward distributions known as $\alpha$-stable distributions. This family refers to a class of distributions parameterized by the exponent $\alpha$, that include the Gaussian ($\alpha=2$), L\'evy ($\alpha = 1/2$) and Cauchy ($\alpha=1$) distributions, all of which are used extensively in economics~\cite{frain2009studies}, finance~\cite{carr2003finite} and signal processing~\cite{shao1993signal}. 

The primary hurdle in creating machine learning algorithms that account for $\alpha$-stable distributions, however, is their intractable probability density, which cannot be expressed analytically. This prevents even a direct evaluation of the likelihood under this distribution. Their heavy-tailed nature, additionally, often leads to standard algorithms (such as Thompson Sampling assuming Gaussian rewards), concentrating on incorrect arms.

In this paper, we create two algorithms for Thompson Sampling under symmetric $\alpha$-stable rewards with finite means. Our contributions can be summarized as follows:
\begin{enumerate}
    \item Using auxiliary variables, we construct a framework for posterior inference in symmetric $\alpha$-stable bandits that leads to the first efficient algorithm for Thompson Sampling in this setting, which we call $\alpha$-TS.
    \item To the best of our knowledge, we provide the first finite-time polynomial bound on the Bayesian Regret of Thompson Sampling achieved by $\alpha$-TS in this setting.
    \item We improve on the regret by proposing a modified Thompson Sampling algorithm, called Robust $\alpha$-TS, that utilizes a truncated mean estimator, and obtain the first $\Tilde{O}(N^{\frac{1}{1+\epsilon}})$ Bayesian Regret in the $\alpha$-stable setting. Our bound matches the optimal bound for $\alpha \in (1, 2)$ (within logarithmic factors).
    \item Through a series of experiments, we demonstrate the proficiency of our two Thompson Sampling algorithms for $\alpha$-stable rewards, which consistently outperform all existing benchmarks.
\end{enumerate}
Our paper is organized as follows: we first give a technical overview of the MAB problem, the Thompson Sampling algorithm and $\alpha$-stable distributions. Next, we provide the central algorithm $\alpha$-TS and its analysis, followed by the same for the Robust $\alpha$-TS algorithm. We then provide experimental results on multiple simulation benchmarks and finally, discuss the related work in this area prior to closing remarks.
\section{Preliminaries}
\subsection{Thompson Sampling}
\textit{The $K$-Armed Bandit Problem:} In any instance of the $K$-armed bandit problem, there exists an agent with access to a set of $K$ actions (or ``arms''). The learning proceeds in rounds, indexed by $t \in [1, T]$. The total number of rounds, known as the \textit{time horizon} $T$, is known in advance. The problem is iterative, wherein for each round $t \in [T]$:
\begin{enumerate}
    \item Agent picks arm $a_t \in \mathcal [K]$.
    \item Agent observes reward $r_{a_t}(t)$ from that arm.
\end{enumerate}
For arm $k \in \mathcal [K]$, rewards come from a distribution $\mathcal D_k$ with mean $\mu_k = \mathbb E_{\mathcal D_k}[r]$\footnote{$\alpha$-stable distributions with $\alpha \leq 1$ do not admit a finite first moment. To continue with existing measures of regret, we only consider rewards with $\alpha > 1$.}. The largest expected reward is denoted by $\mu^* = \max_{k \in [K]} \mu_k$, and the corresponding arm(s) is denoted as the \textit{optimal} arm(s) $k^*$. In our analysis, we will focus exclusively on the i.i.d. setting, that is, for each arm, rewards are independently and identically drawn from $\mathcal D_k$, every time arm $k$ is pulled.

To measure the performance of any (possibly randomized) algorithm we utilize a measure known as \textit{Regret} $R(T)$, which, at any round $T$, is the difference of the cumulative mean reward of the algorithm against the expected reward of always playing an optimal arm.
\begin{equation}
     R(T) = \mu^* T - \sum_{t=0}^T \mu_{a_t}
\end{equation}
\textit{Thompson Sampling (TS):} Thompson Sampling~\cite{thompson1933likelihood} proceeds by maintaining a posterior distribution over the parameters of the bandit arms. If we assume that for each arm $k$, the reward distribution $\mathcal D_k$ is parameterized by a (possibly vector) parameter $\mathbf{\theta}_k$ that come from a set $\bf \Theta$ with a prior probability distribution $p(\theta_k)$ over the parameters, the Thompson Sampling algorithm proceeds by selecting arms based on the posterior probability of the reward under the arms. For each round $t \in [T]$, the agent: 
\begin{enumerate}
\item Draws parameters $\hat\theta_{k} (t)$ for each arm $k \in [K]$ from the posterior distribution of parameters, given the previous rewards ${\bf r}_k(t-1) = \{r_k^{(1)}, r_k^{(2)}, ...\}$ till round $t-1$ (note that the posterior distribution for each arm only depends on the rewards obtained using that arm). When $t=1$, this is just the prior distribution over the parameters.
\begin{equation}
    \hat\theta_k(t) \sim p(\theta_k | {\bf r}_k(t-1)) \propto p({\bf r}_k(t-1) | \theta_k)p(\theta_k)
\end{equation}
\item Given the drawn parameters $\hat\theta_{k}(t)$ for each arm, chooses arm $a_t$ with the largest mean reward over the posterior distribution.
\begin{equation}
    a_t = \argmax_{k \in [K]} \mu_k (\hat\theta_k(t))
\end{equation}
\item Obtains reward $r_t$ after taking action $a_t$ and updates the posterior distribution for arm $a_t$.
\end{enumerate}
In the Bayesian case, the measure for performance we will utilize in this paper is the Bayes Regret (BR)\cite{russo2014learning}, which is the expected regret over the priors. Denoting the parameters over all arms as $\bar\theta = \{\theta_1, ..., \theta_k\}$ and their corresponding product distribution as $\bar{\mathcal D} = \prod_i \mathcal D_i$, for any policy $\pi$, the Bayes Regret is given by:
\begin{equation}
    \text{BayesRegret}(T, \pi) = \mathbb E_{\bar\theta \sim \bar{\mathcal D}}[R(T)].
\end{equation}
While the regret provides a stronger analysis, any bound on the Bayes Regret is essentially a bound on the expected regret, since if an algorithm admits a Bayes Regret of $O(g(T))$, then its Expected Regret is also stochastically bounded by $g(\cdot)$~\cite{russo2014learning}. Formally, we have, for constants $M, \epsilon$:
\begin{equation}
    \mathbb P\left(\frac{\mathbb E[R(T) | \bar\theta]}{g(T)} \geq M\right) \leq \epsilon \ \forall \ T \in \mathbb N.
\end{equation}
\subsection{$\alpha$-Stable Distributions}
$\alpha$-Stable distributions, introduced by L\'evy~\cite{levy1925calcul} are a class of probability distributions defined over $\mathbb R$ whose members are closed under linear transformations.
\begin{definition}[\cite{borak2005stable}]
Let $X_1$ and $X_2$ be two independent instances of the random variable $X$. $X$ is \textbf{stable} if, for $a_1 > 0$ and $a_2 > 0$, $a_1X_1 + a_2X_2$ follows the same distribution as $cX + d$ for some $c > 0$ and $d \in \mathbb R$.
\end{definition}
A random variable $X \sim S_\alpha(\beta, \mu, \sigma)$ follows an $\alpha$-stable distribution described by the parameters $\alpha \in (0, 2]$ (characteristic exponent) and $\beta \in [-1, 1]$ (skewness), which are responsible for the shape and concentration of the distribution, and parameters $\mu \in \mathbb R$ (shift) and $\sigma \in \mathbb R^+$ (scale) which correspond to the location and scale respectively. While it is not possible to analytically express the density function for generic $\alpha$-stable distributions, they are known to admit the characteristic function $\phi(x ; \alpha, \beta, \sigma, \mu)$:
\begin{equation*}
    \phi(x ; \alpha, \beta, \sigma, \mu) = \exp\left\{ ix\mu - |\sigma x|^\alpha \left( 1 - i\beta\sign(x)\Phi_\alpha(x)\right) \right\},
\end{equation*}
where $\Phi_\alpha(x)$ is given by
\begin{equation*}\Phi_\alpha(x) =  \left\{
                \begin{array}{ll}
                  \tan(\frac{\pi\alpha}{2})\text{ when } \alpha \neq 1, &
                  -\frac{2}{\pi}\log|x|,\text{ when }\alpha = 1
                \end{array}
              \right.
\end{equation*}

\begin{figure*}[t]
\centering
\small
\includegraphics[width=0.4\linewidth]{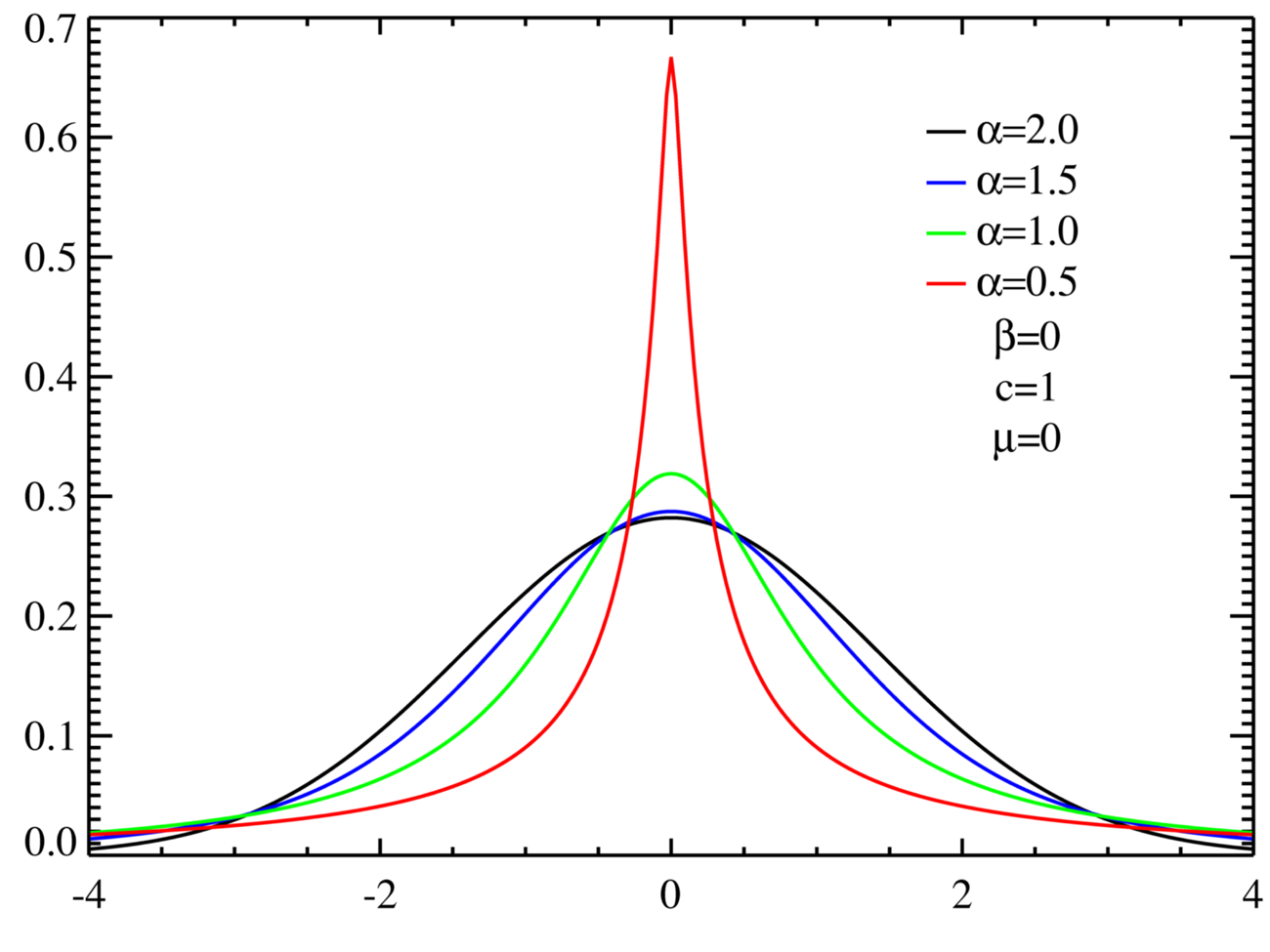}
\caption{Sample probability density for standard ($\mu = 0, \sigma =1$) symmetric $\alpha$-stable distributions with various values of $\alpha$~\cite{wikipedia_image}.}
\label{fig:1}
\end{figure*}
For fixed values of $\alpha, \beta, \sigma$ and $\mu$ we can recover the density function from $\phi(\cdot)$ via the inverse Fourier transform:
\begin{equation*}
    p(z) = \frac{1}{2\pi} \int_{-\infty}^\infty \phi(x ; \alpha, \beta, \sigma, \mu)e^{-izx}dx
\end{equation*}
Most of the attention in the analysis of $\alpha$-stable distributions has been focused on the stability parameter $\alpha$, which is responsible for the tail ``fatness''. It can be shown that asymptotically, the tail behavior ($x \rightarrow \pm \infty$) of $X \sim S_\alpha(\beta, \mu, \sigma)$ follows\cite{borak2005stable}:
\begin{equation*}{\displaystyle f(x)\sim {\frac {1}{|x|^{1+\alpha }}}\left(\sigma^{\alpha }(1+\operatorname {sgn}(x)\beta )\sin \left({\frac {\pi \alpha }{2}}\right){\frac {\Gamma (\alpha +1)}{\pi }}\right)}
\end{equation*}
where $\alpha < 2$ and $\Gamma(\cdot)$ denotes the Gamma function. The power-law relationship admitted by the density is responsible for the heaviness of said tails. 
\begin{fact}[\cite{borak2005stable}]
$X \sim S_\alpha(\beta, \mu, \sigma), \alpha < 2$ admits a moment of order $\lambda$ only when $\lambda \in (-\infty, \alpha)$.
\label{fact:moment}
\end{fact}
From Fact~\ref{fact:moment} it follows that $\alpha$-stable random variables only admit a finite mean for $\alpha > 1$, and also admit a finite variance only when $\alpha = 2$, which corresponds to the family of Gaussian distributions. To continue with existing measures of regret, we restrict our analysis hence to $\alpha$-stable distributions only with $\alpha > 1$. Note that for all our discussions, $1 < \alpha < 2$, hence, all distributions examined are heavy-tailed, with infinite variance. Additionally, we restrict ourselves to only symmetric ($\beta=0$) distributions: asymmetric distributions do not allow a scaled mixture representation (which is the basis of our framework, see Section~\ref{sec:scaled}).
\subsubsection{Sampling from $\alpha$-Stable Densities}
For general values of $\alpha, \beta, \sigma, \mu$, it is not possible to analytically express the density of $\alpha$-stable distributions, and hence we resort to using auxiliary variables for sampling. The Chambers-Mallows-Stuck~\cite{chambers1976method} algorithm is a straightforward method to generate samples from the density $S_\alpha(\beta, 1, 0)$ (for $\alpha \neq 1$) via a non-linear transformation of a uniform random variable $V$ and an exponential random variable $W$, which can then be re-scaled to obtain samples from $S_\alpha(\beta, \sigma, \mu)$ (Algorithm~\ref{alg:cms_generation}).
\begin{algorithm}[t]
\caption{Chambers-Mallows-Stuck Generation}
\label{alg:cms_generation}
\textbf{Input}: $V \sim U(-\pi/2, \pi/2), W \sim E(1)$\\
\textbf{Output}: $X \sim S_\alpha(\beta, \sigma, \mu)$
\begin{algorithmic}[0] 
\STATE Set $B_{\alpha, \beta} = \arctan(\beta \tan(\pi \alpha / 2))\alpha^{-1}$
\STATE Set $S_{\alpha, \beta} = \left( 1 + \beta^2 \tan^2(\pi \alpha / 2)\right)^{1/(2\alpha)}$
\STATE Set $Y = S_{\alpha, \beta} \times \frac{\sin(\alpha(V + B_{\alpha, \beta})}{\cos(V)^{1/\alpha}} \times \left( \frac{\cos(V - \alpha(V + B_{\alpha, \beta})}{W}\right)^{\frac{1-\alpha}{\alpha}}$
\STATE \textbf{return} $X = \sigma Y + \mu$.
\end{algorithmic}
\end{algorithm}
\subsubsection{Products of $\alpha$-Stable Densities}
A central property of $\alpha$-stable densities that will be utilized in the following section is the behavior of products of independent $\alpha$-stable variables.
\begin{lemma}[\cite{borak2005stable}]
Let $Z$ and $Y > 0$ be independent random variables such that $Z \sim S_\gamma(0, \sigma_1, \mu_1)$ and $Y \sim S_\delta(1, \sigma_2, \mu_2)$. Then $ZY^{1/\gamma}$ is stable with exponent $\gamma\delta$.
\label{lemma:product}
\end{lemma}
We now begin with the algorithm description and analysis.
\section{$\alpha$-Thompson Sampling}
We consider the setting where, for an arm $k$, the corresponding reward distribution is given by $\mathcal D_k = S_\alpha(0, \sigma, \mu_k)$ where $\alpha \in (1, 2), \sigma \in \mathbb R^+$ are known in advance, and $\mu_k$ is unknown\footnote{Typically, more general settings have been explored for TS in the Gaussian case, where the variance is also unknown. Our algorithm can also be extended to an unknown scale ($\sigma$) using an inverse Gamma prior, similar to~\cite{godsill1999bayesian}.}. Note that $\mathbb E[r_k] = \mu_k$, and hence we set a prior distribution over the variable $\mu_k$ which is the expected reward for arm $k$. We can see that since the only unknown parameter for the reward distributions is $\mu_k$, $\mathcal D_k$ is parameterized by $\theta_k = \mu_k$. Therefore, the steps for Thompson Sampling can be outlined as follows. For each round $t \in [T]$, the agent: 
\begin{enumerate}
\item Draws parameters $\hat\theta_{k} (t) = \bar\mu_k(t)$ for each arm $k \in [K]$ from the posterior distribution of parameters, given the previous rewards ${\bf r}_k(t-1) = \{r_k^{(1)}, r_k^{(2)}, ...\}$ till round $t-1$.
\begin{equation}
    \bar\mu_k(t) \sim p(\mu_k | {\bf r}_k(t-1)) \propto p({\bf r}_k(t-1) | \mu_k) p(\mu_k)
\end{equation}
\item Given the drawn parameters $\bar\mu_{k}(t)$ for each arm, chooses arm $a_t$ with the largest mean reward over the posterior distribution.
\begin{equation}
    a_t  =  \argmax_{k \in [K]} \bar\mu_k(t)
\end{equation}
\item Obtains reward $r_t$ after taking action $a_t$ and updates the posterior distribution for arm $a_t$.
\end{enumerate}

In this section, we will derive the form of the prior distribution, and outline an algorithm for Bayesian inference.
\subsection{Scale Mixtures of Normals}
\label{sec:scaled}
On setting $\gamma = 2$ (Gaussian) and $\beta = \alpha/2 < 1$ in Lemma~\ref{lemma:product}, the product distribution $X = ZY^{1/2}$ is stable with exponent $\alpha$. This property is an instance of the general framework of \textit{scale mixtures of normals} (SMiN)~\cite{andrews1974scale}, which are described by the following:
\begin{equation}
    p_X(x) = \int_{0}^\infty \mathcal N(x | 0, \lambda \sigma^2) p_\Lambda(\lambda)d\lambda
\end{equation}
This framework contains a large class of heavy-tailed distributions which include the exponential power law, Student's-t and symmetric $\alpha$-stable distributions~\cite{godsill1999bayesian}. The precise form of the variable $X$ depends on the \textit{mixing distribution} $p_\Lambda(\lambda)$. For instance, when $p_\Lambda$ is the inverted Gamma distribution (the conjugate prior for a unknown variance Gaussian), the resulting $p_X$ follows a Student's-t distribution.

Bayesian inference directly from $S_\alpha(0, \sigma, \mu)$ is difficult: the non-analytic density prevents a direct evaluation of the likelihood function, and the non-Gaussianity introduces difficulty in its implementation. However, the SMiN representation enables us to draw samples directly from $S_\alpha(0, \sigma, \mu)$ using the auxiliary variable $\lambda$:
\begin{equation}
    x \sim \mathcal N(\mu, \lambda\sigma^2), \lambda \sim S_{\alpha/2}(1, 1, 0)
\end{equation}
This sampling assists in inference since $x$ is Gaussian conditioned on $\lambda$: given samples of $\lambda$, we can generate $x$ from the induced conditional distribution (which is Gaussian).
\subsection{Posteriors for $\alpha$-Stable Rewards}
Let us examine approximate inference for a particular arm $k \in [K]$. At any time $t \in [T]$, assume this arm has been pulled $n_k(t)$ times previously, and hence we have a vector of reward samples ${\bf r}_k(t) = \{ r_k^{(1)}, ..., r_k^{(n_k(t))}\} $ observed until time $t$. Additionally, assume we have $n_k(t)$ samples of an auxiliary variable ${\boldsymbol \lambda}_{k}(t) = \{ \lambda_{k}^{(1)}, ..., \lambda_{i}^{(n_k(t))}\}$ where $\lambda_k \sim S_{\alpha/2}(1, 1, 0)$. 

Recall that $r_k \sim S_\alpha(0, \sigma, \mu_k)$ for an unknown (but fixed) $\mu_k$. From the SMiN representation, we know that $r_k$ is conditionally Gaussian given the auxiliary variable $\lambda_k$, that is $p(r_k | \lambda_k, \mu_k) \sim \mathcal N(\mu_k, \lambda_k\sigma^2)$. We can then obtain the conditional likelihood as the following:
\begin{equation}p({\bf r}_{k}(t) | {\boldsymbol \lambda}_{k}(t), \mu_k) \propto \exp\left(-\frac{1}{2\sigma^2} \left(\sum_{i=1}^{n_k(t)} \frac{(r_k^{(i)} - \mu_k)^2}{\lambda_k^{(i)}}\right) \right).
\end{equation}
We can now assume a conjugate prior over $\mu_k$, which is a normal distribution with mean $\mu^0_k$ and variance $\sigma^2$. We then obtain the posterior density for $\mu_k$ as (full derivation in the Appendix):
\begin{equation}
    p(\mu_k | {\bf r}_{k}(t), {\boldsymbol \lambda}_{k}(t)) \propto \mathcal N\left(\hat\mu_k(t), \hat\sigma^2_{k}(t)\right) \text{ where, }
    \hat\sigma^2_{k}(t) = \frac{\sigma^2}{\sum_{i=1}^{n_k(t)} \frac{1}{\lambda_k^{(i)}} + 1},\ \hat\mu_{k}(t) = \frac{\sum_{i=1}^{n_k(t)} \frac{r_k^{(i)}}{\lambda_k^{(i)}}+\mu^0_k }{\sum_{i=1}^{n_k(t)} \frac{1}{\lambda_k^{(i)}} + 1}.
\end{equation}
We know that $\hat\sigma^2_k(t) > 0$ since $\lambda_k^{(i)} > 0 \ \forall i$ as they are samples from a positive stable distribution ($\beta = 1$). Given ${\bf r}_{k}(t)$ and $\mu_k$, we also know that the individual elements of ${\boldsymbol \lambda}_{k}(t)$ are independent, which provides us with the following decomposition for the conditional density of ${\boldsymbol \lambda}_{k}(t)$,
\begin{equation}
    p({\boldsymbol \lambda}_{k}(t) | {\bf r}_{k}(t), \mu_k) = \prod_{i=1}^{n_k(t)} p( \lambda_{k}^{(i)} | r_{k}^{(i)}, \mu_k), \text{ where, } 
     p( \lambda_{k}^{(i)} | r_k^{(i)}, \mu_k) \propto \mathcal N(r_{k}^{(i)} | \mu_k, \lambda_{k}^{(i)},\sigma^2) f_{\alpha/2, 1}(\lambda_{k}^{(i)}).
\end{equation}
Here, $f_{\alpha, \beta}( \cdot )$ is the density of a random variable following $S_\alpha(\beta, 1, 0)$. Our stepwise posterior sampling routine is hence as follows. At any time $t$, after arm $k$ is pulled and we receive reward $r_k^{n_k(t)}$, we set ${\bf r}_{k}(t) = [{\bf r}_{k}(t-1), r_k^{n_k(t)}]$. Then for a fixed $Q$ iterations, we repeat:
\begin{enumerate}
    \item For $i \in [1, n_k(t)]$, draw ${\lambda_{k}^{(i)}} \sim p(\lambda_{k}^{(i)} | r_{k}^{(i)}, \mu_k(t))$.
    \item Draw $\mu_k(t) \sim p(\mu_k | {\bf r}_{k}(t), {\boldsymbol \lambda}_{k}(t))$ .
\end{enumerate}
Sampling from the conditional posterior of $\mu_k$ is straightforward since it is Gaussian. To sample from the complicated posterior of $\lambda_k^{(i)}$, we utilize rejection sampling.
\subsection{Rejection Sampling for $\lambda_{k}^{(i)}$}
\label{sec:rejection_sampling}
Sampling directly from the posterior is intractable since it is not analytical. Therefore, to sample $\lambda_{k}^{(i)}$ we follow the pipeline described in~\cite{godsill1999bayesian}. We note that the likelihood of the mean-normalized reward $v_k^{(i)} = r_k^{(i)} - \mu_k (t)$ forms a valid rejection function since it is bounded:
\begin{equation}
    p\left(v_k^{(i)} | 0, \lambda_k^{(i)}\sigma^2\right) \leq  \frac{1}{v_k^{(i)}\sqrt{2\pi}}\exp(-1/2)
\end{equation}
Since $v_k^{(i)} \sim \mathcal N(0; \lambda_k^{(i)}\sigma^2)$. Thus, we get the procedure:
\begin{enumerate}
    \item Draw $\lambda_k^{(i)} \sim S_{\alpha/2}(1, 1, 0)$ (using Algorithm~\ref{alg:cms_generation}).
    \item Draw $u \sim \mathcal U\left(0 , ({v_k^{(i)}}\sqrt{2\pi})^{-1}\exp(-1/2)\right)$.
    \item If $u > p(v_k^{(i)} | 0, \lambda_k^{(i)}\sigma^2)$, reject $\lambda_k^{(i)}$ and go to Step 1.
\end{enumerate}
Combining all these steps, we can now outline our algorithm, $\alpha$-Thompson Sampling ($\alpha$-TS) as described in Algorithm~\ref{alg:alpha_ts}.
\begin{algorithm}[t]
\caption{$\alpha$-Thompson Sampling}
\label{alg:alpha_ts}
\begin{algorithmic}[1] 
\STATE \textbf{Input}: Arms $k \in [K]$, priors $\mathcal N(\mu^0_k, \sigma^2)$ for each arm.
\STATE Set $ D_k = 1, N_k = 0$ for each arm $k$.
\FOR{For each iteration $t \in [1, T]$}
\STATE Draw $\bar\mu_k(t) \sim \mathcal N\left(\frac{\mu^0_k + N_k}{D_k}, \frac{\sigma^2}{D_k}\right)$ for each arm $k$.
\STATE Choose arm $A_t = \displaystyle\argmax_{k \in [K]} \bar\mu_k(t)$, and get reward $r_t$.
\FOR{$q \in [0, Q)$}
\STATE Calculate $v_{A_t}^{(t)} = r_t - \bar\mu_{A_t}$.
\STATE Draw $\lambda_k^{(t)}$ following Section~\ref{sec:rejection_sampling}.
\STATE Set $D_q = D_k + 1/\lambda_k^{(t)}, N_q = N_k + r_t/\lambda_k^{(t)}$.
\STATE Draw $\bar\mu_{A_t} \sim \mathcal N\left(\frac{\mu^0_k + N_q}{D_q}, \frac{\sigma^2}{D_q}\right)$.
\ENDFOR
\STATE Set $D_k = D_k + 1/\lambda_k^{(t)}, N_k = N_k + r_t/\lambda_k^{(t)}$.
\ENDFOR
\end{algorithmic}
\end{algorithm}

It is critical to note that in Algorithm~\ref{alg:alpha_ts}, we do not draw from the full vector of ${\boldsymbol \lambda}_{k}(t)$ at every iteration, but only from the last obtained reward. This is done to accelerate the inference process, and while it leads to a slower convergence of the posterior, we observe that it performs well in practice. Alternatively, one can re-sample ${\boldsymbol \lambda}_{k}(t)$ over a fixed window of the previous rewards, to prevent the runtime from increasing linearly while enjoying faster convergence.
\subsection{Regret Analysis}
In this section, we derive an upper bound on the finite-time Bayesian Regret (BR) incurred by the $\alpha$-TS algorithm. We continue with the notation used in previous sections, and assume a $K$ armed bandit with $T$ maximum trials. Each arm $k$ follows an $\alpha$-stable reward $S_\alpha(0, \sigma, \mu_k)$, and without loss of generality, let $\mu^* = \max_{k \in [K]} \mu_i$ denote the arm with maximum mean reward. 
\begin{theorem}[Regret Bound]
Let $K> 1, \alpha \in (1, 2), \sigma \in \mathbb R^+, \mu_{k: k \in [K]} \in \mathbb [-M, M]$. For a $K$-armed bandit with rewards for each arm $k$ drawn from $S_\alpha(\beta, \sigma, \mu_k)$, we have, asymptotically, for $\epsilon$ chosen a priori such that $\epsilon \rightarrow (\alpha - 1)^-$,
\begin{equation*}
    \text{BayesRegret}(T, \pi^{TS}) = O(K^{\frac{1}{1+\epsilon}}T^{\frac{1+\epsilon}{1+2\epsilon}})
\end{equation*}
\label{thm:alpha_ts}
\end{theorem}
\textit{Proof-sketch.} We first utilize the characteristic function representation of the probability density to obtain the centered $(1+\epsilon)$ moment of the reward distribution for each arm. We use this moment to derive a concentration bound on the deviation of the empirical mean of $\alpha$-stable densities. Next, we proceed with the decomposition of the Bayesian Regret in terms of upper-confidence bounds, as done in~\cite{russo2014learning}. Using our concentration result, we set an upper confidence bound at any time $t$ for any arm $k$, and then obtain the result via algebraic manipulation. The complete proof has been omitted here for brevity, but is included in detail in the appendix.

We note the following: First, the only additional assumption we make on the reward distributions is that the true means are bounded, which is a standard assumption~\cite{russo2014learning} and easily realizable in most application cases. Next, given the established polynomial deviations of the empirical mean, obtanining a polylogarithmic regret is not possible, and the selection of $\epsilon$ governs the finite-time regret. As $\epsilon \rightarrow \alpha - 1$, we see that while the growth of $T^{\frac{1+\epsilon}{1+2\epsilon}}$ decreases, the constants in the finite-time expression grow, and are not finite at $\epsilon = \alpha -1$. This behavior arises from the non-compactness of the set of finite moments for $\alpha$-stable distributions (see Appendix for detailed analysis).

Compared to the problem-independent regret bound of $O(\sqrt{KT\log T})$ for Thompson Sampling on the multi-armed Gaussian bandit problem demonstrated by~\cite{agrawal2013further}, our bound differs in two aspects: first, we admit a $K^{\frac{1}{1+\epsilon}}$ complexity on the number of arms, which contrasted with the Gaussian bandit is identical when $\epsilon \rightarrow 1$. Second, we have a superlinear dependence of order $T^{\frac{1+\epsilon}{1+2\epsilon}}$. Compared to the Gaussian case, we see that we obtain a polynomial regret instead of a polylogarithmic regret, which can be attributed to the polynomial concentration of heavy-tailed distributions.

In the next section, we address the issue of polynomial concentration by utilizing the more robust, truncated mean estimator instead of the empirical mean, and obtain a modified, robust version of $\alpha$-TS.
\begin{algorithm}[t]
\caption{Robust $\alpha$-Thompson Sampling}
\label{alg:robust_alpha_ts}
\begin{algorithmic}[1] 
\STATE \textbf{Input}: Arms $k \in [K]$, priors $\mathcal N(\mu^0_k, \sigma^2)$ for each arm.
\STATE Set $ D_k = 1, N_k = 0$ for each arm $k$.
\FOR{For each iteration $t \in [1, T]$}
\STATE Draw $\bar\mu_k(t) \sim \mathcal N\left(\frac{\mu^0_k + N_k}{D_k}, \frac{\sigma^2}{D_k}\right)$ for each arm $k$.
\STATE Choose arm $A_t = \argmax_k \bar\mu_k(t)$, and get reward $r_t$.
\FOR{$q \in [0, Q)$}
\STATE Calculate $v_{A_t}^{(t)} = r_t - \bar\mu_{A_t}$.
\STATE Sample $\lambda_k^{(t)}$ following Section~\ref{sec:rejection_sampling}.
\STATE Set $D_q = D_k + 1/\lambda_k^{(t)}, N_q = N_k + r_t/\lambda_k^{(t)}$.
\STATE Sample $\bar\mu_{A_t} \sim \mathcal N\left(\frac{\mu^0_k + N_q}{D_q}, \frac{\sigma^2}{D_q}\right)$.
\ENDFOR
\STATE Set $D_k = D_k + 1/\lambda_k^{(t)}, N_k = N_k + r_t/\lambda_k^{(t)}$.
\STATE Recompute robust mean.
\ENDFOR
\end{algorithmic}
\end{algorithm}
\subsection{Robust $\alpha$-Thompson Sampling}
Assume that for all arms $k \in [K], \mu_k \leq M$. Note that this assumption is equivalent to the boundedness assumption in the analysis of $\alpha$-TS, and is a reasonable assumption to make in any practical scenario with some domain knowledge of the problem. Let $\delta \in (0, 1), \epsilon \in (0, \alpha-1)$. Now, consider the truncated mean estimator $\hat r_k^*(t)$ given by:
\begin{multline}
    \hat r_k^*(t) = \frac{1}{n_k(t)}\sum_{i=1}^{n_k(t)}r_k^{(i)}\mathbbm{1}\left\{|r_k^{(i)}| \leq \left(\frac{H(\epsilon, \alpha, \sigma)\cdot i}{2\log(T)}\right)^{\frac{1}{1+\epsilon}} \right\} \\
    \text{where, } H(\epsilon, \alpha, \sigma) = \left(\frac{\epsilon\left(M \cdot \Gamma(-\epsilon/\alpha)+ \sigma\alpha\Gamma(1 - \frac{\epsilon + 1}{\alpha})\right)}{\sigma\alpha\sin\left(\frac{\pi\cdot\epsilon}{2}\right)\Gamma(1-\epsilon)}\right)
    \label{eqn:truncated}
\end{multline}
$\hat r^*_k(t)$ then describes a truncated mean estimator for an arm $k$ (pulled $n_k(t)$ times), where a reward $r_k^{(i)}$ at any trial $i$ of the arm is discarded if it is larger than the bound. Intuitively, we see that this truncated mean will prevent outliers from affecting the posterior. We choose such a form of the truncation since it allows us to obtain an exponential concentration for $\alpha$-stable densities, which will assist us in obtaining polylogarithmic regret.

As can be seen, the corresponding Robust $\alpha$-TS algorithm is identical to the basic $\alpha$-TS algorithm, except for this step of rejecting a reward (and replacing it with 0) and is outlined in Algorithm~\ref{alg:robust_alpha_ts}. We now describe the regret incurred by this modified $\alpha$-TS algorithm.
\begin{theorem}[Regret Bound]
Let $K> 1, \alpha \in (1, 2), \sigma \in \mathbb R^+, \mu_{k: k \in [K]} \in \mathbb [-M, M]$. For a $K$-armed bandit with rewards for each arm $k$ drawn from $S_\alpha(\beta, \sigma, \mu_k)$, we have, for $\epsilon$ chosen a priori such that $\epsilon \rightarrow (\alpha - 1)^-$ and truncated estimator given in Equation~(\ref{eqn:truncated}), 
\begin{equation*}
    \text{BayesRegret}(T, \pi^{RTS}) = \Tilde{O}\left((KT)^{\frac{1}{1+\epsilon}}\right)
\end{equation*}
\label{theorem:regret_log}
\end{theorem}
\textit{Proof-sketch.} The truncated mean estimator can be used to obtain an exponential concentration bound on the empirical mean. This can be understood intuitively as by rejecting values that are very far away from the mean, our empirical mean is robust to outliers, and would require less samples to concentrate around the true mean. Using the concentration result, we follow an analysis identical to Theorem~\ref{thm:alpha_ts}, with exponential concentration in the upper confidence bounds. The full proof is deferred to the Appendix for brevity.

We see that this bound is tight: when $\alpha = 2$ (Gaussian), $\epsilon$ can be set to $1$, and we see that this matches the upper bound of $O(\sqrt{KT\log T})$\cite{russo2014learning} for the Gaussian case, which has also been shown to be optimal~\cite{agrawal2013further}. Algorithm~\ref{alg:robust_alpha_ts} is hence more robust, with performance improvements increasing as $\alpha$ decreases: the likelihood of obtaining confounding outliers increases as $\alpha \rightarrow 1$, and can perturb the posterior mean in the naive $\alpha$-TS algorithm. In the next section, we discuss some experimental results that cement the value of $\alpha$-TS for $\alpha$-stable bandits.
\section{Experiments}
We conduct simulations with the following benchmarks -- (i) an $\varepsilon$-greedy agent with linearly decreasing $\varepsilon$, (ii) Regular TS with Gaussian priors and a Gaussian assumption on the data (Gaussian-TS), (iii) Robust-UCB~\cite{bubeck2013bandits} for heavy-tailed distributions using a truncated mean estimator, and (iv) $\alpha$-TS and (v) Robust $\alpha$-TS, both with $Q$(iterations for sampling) as $50$.\\

\textbf{Setting Priors for TS:} In all the Thompson Sampling benchmarks, setting the priors are crucial to the algorithm's performance. In the competitive benchmarking, we randomly set the priors for each arm from the same range we use for setting the mean rewards.

\subsection{Performance against Competitive Benchmarks}
\begin{figure*}[t]
\centering
\small
\includegraphics[width=\linewidth]{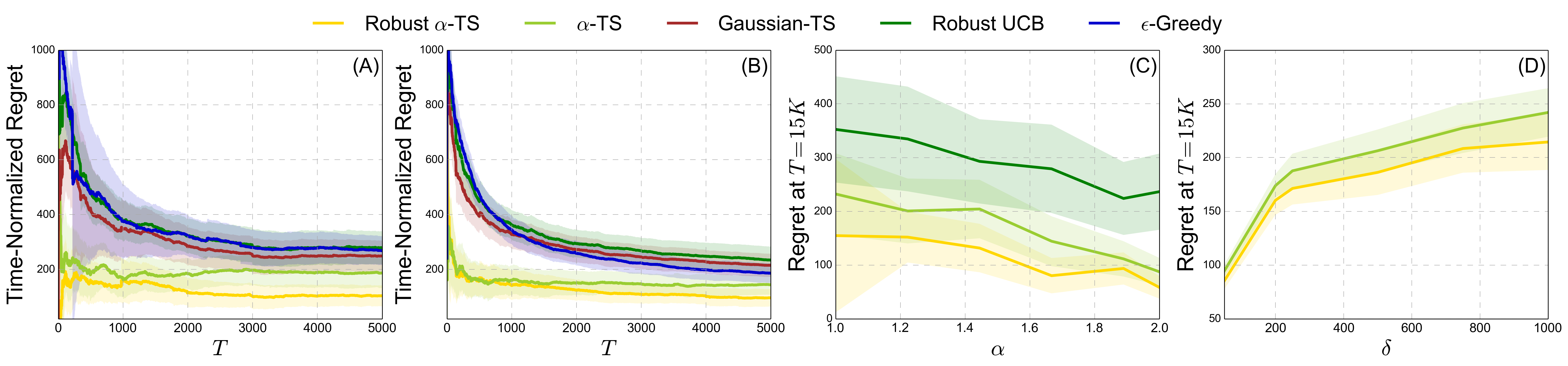}
\caption{(A) Competitive benchmarking for $\alpha=1.3$, and (B) $\alpha=1.8$; (C) Ablation studies for varying $\alpha$, and (D) varying prior strength. Shaded areas denote variance scaled by $0.25$ in (A) and (B), and scaled by $0.5$ in (C) and (D). Figures are best viewed in color.}
\label{fig:1}
\end{figure*}
We run 100 MAB experiments each for all 5 benchmarks for $\alpha=1.8$ and $\alpha=1.3$, and $K=50$ arms, and for each arm, the mean is drawn from $[0, 2000]$ randomly for each experiment, and $\sigma = 2500$. Each experiment is run for $T=5000$ iterations, and we report the regret averaged over time, i.e. $\mathbb R(t)/t$ at any time $t$. 

In Figures~\ref{fig:1}A and ~\ref{fig:1}B, we see the results of the regret averaged over all 100 experiments. We observe that for $\alpha=1.3$, there are more substantial variations in the regret (low $\alpha$ implies heavy outliers), yet both algorithms comfortably outperform the other baselines. 

In the case of $\alpha=1.8$, the variations are not that substantial, the performance follows the same trend. It is important to note that regular Thompson Sampling (Gaussian-TS) performs competitively, although in our experiments, we observed that when $K$ is large, the algorithm often concentrates on the incorrect arm, and subsequently earns a larger regret.

Intuitively, we can see that whenever arms have mean rewards close to each other (compared to the variance), $\epsilon$-greedy will often converge to the incorrect arm. Robust-UCB, however, makes very weak assumptions on the data distributions, and hence has a much larger exploration phase, leading to larger regret. Compared to $\alpha$-TS, Robust $\alpha$-TS is less affected by outlying rewards (as is more evident in $\alpha=1.3$ vs. $\alpha=1.8$) and hence converges faster.
\subsection{Ablation Studies}
We additionally run two types of ablation studies - first, we compare the performance of $\alpha$-TS on the identical set up as before (same $K$ and reward distributions), but with varying values of $\alpha \in (1, 2)$. We report the expected time-normalized regret averaged over 100 trials in Figure~\ref{fig:1}C, and observe that (i) as $\alpha$ increases, the asymptotic regret decreases faster, and (ii) as expected, for lower values of $\alpha$ there is a substantial amount of variation in the regret. 

Secondly, we compare the effect of the sharpness of the priors. In the previous experiments, the priors are drawn randomly from the same range as the means, without any additional information. However, by introducing more information about the means through the priors, we can expect better performance. In this experiment, for each mean $\mu_k$, we randomly draw the prior mean $\mu^0_k$ from $[\mu_k-\delta, \mu_k+\delta]$, and observe the regret after $T=15K$ trials for $\delta$ from $50$ to $1000$. The results for this experiment are summarized in Figure~\ref{fig:1}D for $K=10$ and $\sigma=25$, and results are averaged over 25 trials each. We see that with uninformative priors, $\alpha$-TS performs competitively, and only gets better as the priors get sharper.
\section{Related Work}
A version of the UCB algorithm~\cite{auer2002finite} has been proposed in~\cite{bubeck2013bandits} coupled with several robust mean estimators to obtain Robust-UCB algorithms with optimal \textit{problem-dependent} (i.e. dependent on individual $\mu_k$s) regret when rewards are heavy-tailed. However, the optimal version of their algorithm has a few shortcomings that $\alpha$-TS addresses: first, their algorithm requires the median-of-means estimator, which has a space complexity of $O(\log \delta^{-1})$ and time complexity of $O(\log \log \delta^{-1})$ per update, where $\delta$ is the confidence factor. This makes the algorithm expensive as the confidence is increased (for accurate prediction). Second, their regret bound has a dependence of $(\mu^*-\mu_k)^{1/(1-\alpha)}$, which can become arbitrarily large for a stable distribution with $\alpha$ close to 1. Finally, there is no mechanism to incorporate prior information, which can be advantageous, as seen even with weak priors.~\cite{vakili2013deterministic} introduce a deterministic exploration-exploitation algorithm, which achieves same order regret as Robust-UCB for heavy-tailed rewards.

There has been work in analysing Thompson Sampling for specific Pareto and Weibull heavy-tailed distributions in~\cite{korda2013thompson}, however, they do not provide an efficient posterior update rule and rely on approximate inference under the Jeffrey's prior. More importantly, the Weibull and Pareto distributions typically have ``lighter'' tails owing to the existence of more higher order moments, and hence cannot typically be used to model very heavy tailed signals.

In related problems,~\cite{yupure} provide a purely exploratory algorithm for best-arm identification under heavy-tailed rewards, using a finite $(1+\epsilon)^{th}$ moment assumption. Similarly,~\cite{shao2018almost, medina2016no} explore heavy-tailed rewards in the linear bandit setting. 
\section{Conclusion}
In this paper, we first designed a framework for efficient posterior inference for the $\alpha$-stable family, which has largely been ignored in the bandits literature owing to its intractable density function. We formulated the first polynomial problem-independent regret bounds for Thompson Sampling with $\alpha$-stable densities, and subsequently improved the regret bound to achieve the optimal regret identical to the sub-Gaussian case, providing an efficient framework for decision-making for these distributions.

Additionally, our intermediary concentration results provide a starting point for other machine learning problems that may be investigated in $\alpha$-stable settings. There is ample evidence to support the existence of $\alpha$-stability in various modeling problems across economics~\cite{frain2009studies}, finance~\cite{bradley2003financial} and behavioral studies~\cite{mahanti2013tale}. With tools such as ours, we hope to usher scientific conclusions in problems that cannot make sub-Gaussian assumptions, and can lead to more robust empirical findings. Future work may include viewing more involved decision-making processes, such as MDPs, in the same light, leading to more (distributionally) robust algorithms.\\

\textbf{Acknowledgements:} The authors would like to sincerely thank Peter Krafft, Yan Leng, Dhaval Adjodah, Ziv Epstein, Michiel Bakker, Ishaan Grover and Spandan Madan for their comments and suggestions for this manuscript, and MIT Quest for Intelligence for their generous support with computational resources. We also would like to acknowledge Tor Lattimore and Csaba Szepesvari's text on bandit algorithms~\cite{lattimore2018bandit} for their excellent treatment of bandit algorithms.
\nocite{shao1993signal}
\newpage
\setcounter{theorem}{0}
\setcounter{lemma}{0}
\setcounter{definition}{0}
\section*{Appendix}
\subsection*{Posterior Distribution of $\mu_k$}
At any time $t \in [T]$, arm $k \in [K]$, assume this arm has been pulled $n_k(t)$ times previously, and hence we have a vector of reward samples ${\bf r}_k(t) = \{ r_k^{(1)}, ..., r_k^{(n_k(t))}\} $ observed until time $t$. Additionally, assume we have $n_k(t)$ samples of an auxiliary variable ${\boldsymbol \lambda}_{k}(t) = \{ \lambda_{k}^{(1)}, ..., \lambda_{i}^{(n_k(t))}\}$ where $\lambda_k \sim S_{\alpha/2}(1, 1, 0)$. 
\begin{figure*}[h]
\centering
\small
\includegraphics[width=0.25\linewidth]{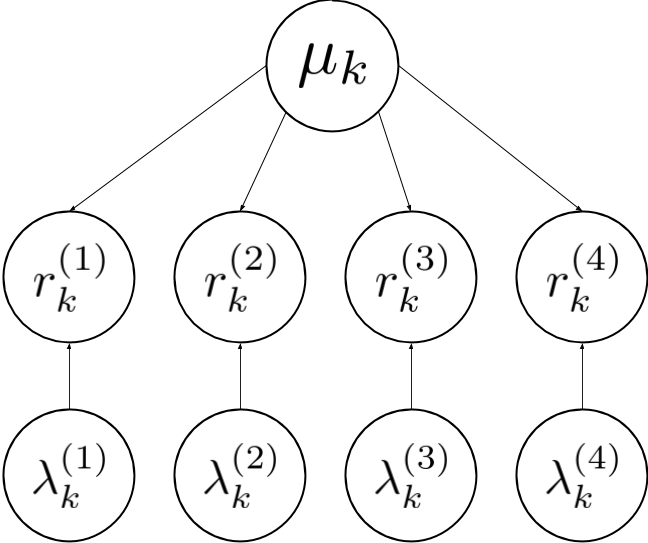}
\caption{Directed graphical model for arm $k$ where $n_k(t) = 4$.}
\label{fig:graphical_model}
\end{figure*}
Recall that $r_k \sim S_\alpha(0, \sigma, \mu_k)$ for an unknown (but fixed) $\mu_k$. From the SMiN representation, we know that $r_k$ is conditionally Gaussian given the auxiliary variable $\lambda_k$, that is $p(r_k | \lambda_k, \mu_k) \sim \mathcal N(\mu_k, \lambda_k\sigma^2)$. We can then obtain the conditional likelihood as
\begin{equation}
p({\bf r}_{k}(t) | {\boldsymbol \lambda}_{k}(t), \mu_k) \propto \exp\left(\frac{1}{2\sigma^2} \left(\sum_{i=1}^{n_k(t)} \frac{(r_k^{(i)} - \mu_k)^2}{\lambda_k^{(i)}}\right) \right).
\end{equation}
We can now assume a conjugate prior over $\mu_k$, which is a normal distribution with mean $\mu^0_k$ and variance $\sigma^2$. By Bayes' rule, we know that
\begin{align}
    p(\mu_k | {\boldsymbol \lambda}_{k}(t), {\bf r}_{k}(t)) &\propto p\left({\bf r}_{k}(t) | \mu_k, {\boldsymbol \lambda}_{k}(t)\left)p\right(\mu_k | {\boldsymbol \lambda}_{k}(t)\right) \\
    &\stackrel{(a)}{\propto} p({\bf r}_{k}(t) | \mu_k, {\boldsymbol \lambda}_{k}(t))p(\mu_k) \\
    &\propto \exp\left(-\frac{1}{2\sigma^2} \left(\sum_{i=1}^{n_k(t)} \frac{(r_k^{(i)} - \mu_k)^2}{\lambda_k^{(i)}}\right) \right) \cdot \exp\left(-\frac{(\mu_k - \mu_k^0)^2}{2\sigma^2} \right) \\
    &\propto \exp\left(-\frac{1}{2\sigma^2} \left(\sum_{i=1}^{n_k(t)} \frac{(r_k^{(i)} - \mu_k)^2}{\lambda_k^{(i)}} + (\mu_k - \mu_k^0)^2\right) \right)
\end{align}
\begin{align}
   p(\mu_k | {\boldsymbol \lambda}_{k}(t), {\bf r}_{k}(t)) &\propto \exp\left(-\frac{1}{2\sigma^2} \left(\mu_k^2\left(\sum_{i=1}^{n_k(t)}\frac{1}{\lambda_k^{(i)}} +1\right) -2\mu_k\left(\sum_{i=1}^{n_k(t)}\frac{r_k^{(i)}}{\lambda_k^{(i)}} + \mu_k^0\right)\right) \right) \\
   &\propto \exp\left(-\frac{1}{2\sigma^2\left(\sum_{i=1}^{n_k(t)}\frac{1}{\lambda_k^{(i)}} +1\right)^{-1}} \left(\mu_k^2 -2\mu_k\left(\frac{\left(\sum_{i=1}^{n_k(t)}\frac{r_k^{(i)}}{\lambda_k^{(i)}} + \mu_k^0\right)}{\left(\sum_{i=1}^{n_k(t)}\frac{1}{\lambda_k^{(i)}} +1\right)}\right)\right) \right) \\
   &\propto \exp\left(-\frac{\left(\mu_k - \frac{\left(\sum_{i=1}^{n_k(t)}\frac{r_k^{(i)}}{\lambda_k^{(i)}} + \mu_k^0\right)}{\left(\sum_{i=1}^{n_k(t)}\frac{1}{\lambda_k^{(i)}} +1\right)}\right)^2}{2\sigma^2\left(\sum_{i=1}^{n_k(t)}\frac{1}{\lambda_k^{(i)}} +1\right)^{-1}} \right).
\end{align}
Where $(a)$ follows from the independence of $\lambda_k^{(t)}$ and $\mu_k$. Comparing the above to a standard normal distribution, we then obtain the posterior density for $\mu_k$ as
\begin{multline}
    p(\mu_k | {\bf r}_{k}(t), {\boldsymbol \lambda}_{k}(t)) \propto \mathcal N\left(\hat\mu_k(t), \hat\sigma^2_{k}(t)\right) \text{ where, }
    \hat\sigma^2_{k}(t) = \frac{\sigma^2}{\sum_{i=1}^{n_k(t)} \frac{1}{\lambda_k^{(i)}} + 1},\ \hat\mu_{k}(t) = \frac{\sum_{i=1}^{n_k(t)} \frac{r_k^{(i)}}{\lambda_k^{(i)}}+\mu^0_k }{\sum_{i=1}^{n_k(t)} \frac{1}{\lambda_k^{(i)}} + 1}.
\end{multline}
\subsection*{Regret Analysis: $\alpha$-Thompson Sampling}
To derive a bound on the regret of $\alpha$-Thompson Sampling, we first state a few elementary properties of $\alpha$-stable distributions.
\begin{lemma}[\cite{borak2005stable}]
If $X \sim S_\alpha(0, \sigma_1, \mu_1)$ and $Y \sim S_\alpha(0, \sigma_2, \mu_2)$, then $X + Y \sim S_\alpha(0, \sigma, \mu)$, where,
\begin{equation*}
    \sigma  = \left(\sigma_1^\alpha + \sigma_2^\alpha\right)^{1/\alpha}, \mu = \mu_1 + \mu_2
\end{equation*}
\label{lemma:stable_sums}
\end{lemma}
\begin{lemma}[\cite{borak2005stable}]
If $X \sim S_\alpha(0, \sigma, \mu)$, then for $a \neq 0, b \in \mathbb R$,
\begin{equation*}
    aX + b \sim S_\alpha(0, |a|\sigma, a\mu + b)
\end{equation*}
\label{lemma:stable_scales}
\end{lemma}
By composing these two results, we can obtain a representation for the empirical mean of a sequence of i.i.d. samples from an $\alpha$-stable distribution, as described next.
\begin{lemma}
Let $X_1, X_2, ..., X_n$ be an i.i.d. sequence of $n$ random variables following $S_\alpha(0, \sigma, \mu)$, and $\bar X$ denote the empirical mean:
\begin{equation*}
    \bar X = \frac{1}{n}\sum_{i=1}^n X_i
\end{equation*}
Then $\bar X \sim S_\alpha(0, \sigma^*, \mu^*)$, where,
\begin{equation*}
    \sigma^* = \sigma n^{(\frac{1}{\alpha}-1)}, \mu^* = \mu
\end{equation*}
\label{lemma:stable_average}
\end{lemma}
\begin{proof}
Let $Y = \sum_{i=1}^n X_i$. Then $Y \sim S_\alpha(\sigma n^{1/\alpha}, \mu)$, by Lemma~\ref{lemma:stable_sums}. Then, $\bar X  = Y/n$, and the result follows from Lemma~\ref{lemma:stable_scales}.
\end{proof}
Next, we state an important result that allows us to obtain the moments of symmetric, zero-mean $\alpha$-stable distributions.
\begin{lemma}[Theorem 4 of \cite{shao1993signal}]
For $X \sim S_\alpha(0, \sigma, 0), p \in (0, \alpha)$,
\begin{equation*}
    \mathbb E[|X|^p] = C(p, \alpha)\sigma^{\frac{p}{\alpha}} \text{ where, } C(p, \alpha) = \frac{2^{p+1} \Gamma(\frac{p+1}{2})\Gamma(-p/\alpha)}{\alpha \sqrt{\pi} \Gamma(-p/2)}, \text{and } \Gamma(x) = \int_0^{\infty} t^{x-1}e^{-t}dt.
\end{equation*}
\label{lemma:stable_expectation}
\end{lemma}
The next two results allow us to obtain a concentration bound of the empirical of a fininte number of samples of a symmetric $\alpha$-stable distribution.
\begin{lemma}[Lemma 3 of~\cite{bubeck2013bandits}]
Let $X_1, ..., X_n$ be an i.i.d. sequence of random variables with finite mean $\mu$, finite $(1+\epsilon)$ centered moments $\mathbb E[|X-\mu|^{1+\epsilon}] \leq v_\epsilon$ and finite raw moments $\mathbb E[|X|^{1+\epsilon}] \leq u_\epsilon$, for $\epsilon \in (0, 1]$. Let $\hat{\mu}$ denote the empirical mean:
\begin{equation*}
    \hat\mu = \frac{1}{n}\sum_{i=1}^n X_i
\end{equation*}
Then, for any $\delta \in (0, 1)$, we have, with probability at least $1-\delta$,
\begin{equation*}
    \hat\mu \leq \mu + \left(\frac{3v_\epsilon}{\delta n^\epsilon}\right)^{\frac{1}{1+\epsilon}}
\end{equation*}
\label{lemma:concentration}
\end{lemma}
\begin{lemma}
Let $X_1, ..., X_n$ be an i.i.d. sequence of random variables following $S_\alpha(0, \sigma, \mu)$ for $\alpha \in (1, 2)$. Let $\hat{\mu}$ denote the empirical mean:
\begin{equation*}
    \hat\mu = \frac{1}{n}\sum_{i=1}^n X_i
\end{equation*}
Then, for any $\delta \in (0, 1)$ and $\epsilon \in (0, \alpha - 1)$, we have, with probability at least $1-\delta$,
\begin{equation*}
    |\hat\mu - \mu| \leq  \sigma^{\frac{1}{\alpha}} \left(\frac{2C(1+\epsilon, \alpha)}{\delta n_k(t)^\epsilon}\right)^{\frac{1}{1+\epsilon}}
\end{equation*}
\label{lemma:stable_concentration}
\end{lemma}
\begin{proof}
From Lemma~\ref{lemma:stable_scales}, we know that if $X \sim S_\alpha(0, \sigma, \mu)$ then $X-\mu \sim S_\alpha(0, \sigma, 0)$. Applying this in Lemma~\ref{lemma:stable_expectation}, we get  $v_\epsilon = C(1+\epsilon, \alpha)\sigma^{\frac{1+\epsilon}{\alpha}}$. Note that the raw moments of order $1+\epsilon < \alpha$ are also finite for $\alpha$-stable densities, therefore $u_\epsilon < \infty \ \forall \epsilon \in (0, \alpha-1)$ and Lemma~\ref{lemma:concentration} can be applied. Thus, an application of Lemma~\ref{lemma:concentration} for both tails with probability $\delta/2$ gets the desired result.
\end{proof}
A few remarks are in order. First, we note that this bound is tight, and while this introduces polynomial deviations for the empirical mean, it is not an artefact of the proof method, but the heaviness of the tails themselves~\cite{bubeck2013bandits}.

Using this concentration result, we can now proceed to proving the regret bound. Our analysis is inspired by the sub-Gaussian case analysed in~\cite{russo2014learning}, and we prove a few important results within the main proof.
\begin{theorem}[Regret Bound]
Let $K> 1, \alpha \in (1, 2), \sigma \in \mathbb R^+, \mu_{i: i \in [K]} \in \mathbb [-M, M]$. For a $K$-armed bandit with rewards for each arm drawn from $S_\alpha(0, \sigma, \mu_i)$, we have, for finite-time $T$:
\begin{equation*}
    \text{BR}(T, \pi^{TS}) \leq \inf_{\epsilon < \alpha -1} \left\{2KMT + 2\left(g(\epsilon, \alpha)\right)^{\frac{1}{1+\epsilon}} K^{\frac{1}{1+\epsilon}}T^{\frac{1+\epsilon}{1+2\epsilon}}\right\}
\end{equation*}
Asymptotically, for $\epsilon$ chosen a priori such that $\epsilon \rightarrow (\alpha - 1)^-$, 
\begin{equation*}
    \text{BayesRegret}(T, \pi^{TS}) = O(K^{\frac{1}{1+\epsilon}}T^{\frac{1+\epsilon}{1+2\epsilon}})
\end{equation*}
\label{theorem:regret}
\end{theorem}

\begin{proof}
Consider a $K$-armed bandit with rewards for arm $k$ drawn from $S_\alpha(0, \sigma, \mu_k)$. Let $n_{k}(t)$ denote the number of times arm $k$ has been pulled until time $t$. Then $t-1 = \sum_{k=1}^K n_k(t)$. Let us denote the empirical average reward for arm $k$ up to (and including) time $t-1$ as $\hat{r}_{k}(t)$, and denote the arm pulled at any time $t$ as $a_t$, and the optimal arm as $a^*_t$. We then set an upper  confidence bound for arm $k$ at any time $t$ as 
\begin{equation}
    U_k(t) = \text{clip}_{[-M, M]}\left[\hat{r}_{k}(t) + \sigma^{\frac{1}{\alpha}} \left(\frac{2C(1+\epsilon, \alpha)}{\delta n_k(t)^\epsilon}\right)^{\frac{1}{1+\epsilon}}\right] 
    \label{eqn:ucb_basic}
\end{equation} 
for $0 < \epsilon < 1, M > 0$. Let $E$ be the event when for all $k \in [K]$ arms, over all iterations $t \in [T]$, we have:
\begin{equation}
    |\hat{r}_k(t) - \mu_k | \leq  \sigma^{\frac{1}{\alpha}} \left(\frac{2C(1+\epsilon, \alpha)}{\delta n_k(t)^\epsilon}\right)^{\frac{1}{1+\epsilon}}.
    \label{eqn:event_e_basic}
\end{equation}
\begin{lemma}
For the setup described above, we have, for event $E$ and $\delta \in (0, 1)$,
\begin{equation*}
    \mathbb P(E^c) \leq KT\delta.
\end{equation*}
\label{lemma:stable_product_comp}
\end{lemma}
\begin{proof}
From Lemma~\ref{lemma:stable_concentration}, for arm $k \in [K]$ at time $t \in [T]$:
\begin{equation*}
    \mathbb P\left(|\hat{r}_k(t) - \mu_k| \leq \sigma^{\frac{1}{\alpha}} \left(\frac{2C(1+\epsilon, \alpha)}{\delta n_k(t)^\epsilon}\right)^{\frac{1}{1+\epsilon}}\right) \geq 1-\delta
\end{equation*}
The event $E^c$ holds whenever the bound is violated for at least one arm $k$ at one instance $t$. Therefore,
\begin{align*}
    \mathbb P(E^c) &\leq \mathbb P\left( \bigcup_{\substack{k=1\\t=1}}^{K, T} \left\{|\hat{r}_k(t-1) - \mu_k| > \sigma^{\frac{1}{\alpha}} \left(\frac{2C(1+\epsilon, \alpha)}{\delta n_k(t)^\epsilon}\right)^{\frac{1}{1+\epsilon}}\right\}\right) \\
    &\stackrel{(a)}{\leq} \sum_{\substack{k=1\\t=1}}^{K, T} \mathbb P\left(|\hat{r}_k(t-1) - \mu_k| > \sigma^{\frac{1}{\alpha}} \left(\frac{2C(1+\epsilon, \alpha)}{\delta n_k(t)^\epsilon}\right)^{\frac{1}{1+\epsilon}}\right) \\
    &\stackrel{(b)}{\leq} KT\delta.
\end{align*}
Where $(a)$ is an application of the union bound, and $(b)$ is obtained using Lemma~\ref{lemma:stable_concentration}.
\end{proof}

Using this lemma, we can now prove Theorem~\ref{theorem:regret}. We begin with the seminal result of~\cite{russo2014learning}.
\begin{lemma}[Proposition 1 of~\cite{russo2014learning}]
Let $\pi^{TS}$ be any policy followed by Thompson Sampling. For any sequence of upper confidence bounds $\{U_t | t \in \mathbb N\}$,
\begin{equation*}
    \text{BayesRegret}(T, \pi^{TS}) = \mathbb E\left[\sum_{t=1}^T\left(U_{a_t}(t) - \mu_{a_t}\right)\right] + \mathbb E\left[\sum_{t=1}^T\left(\mu_{a^*_t} - U_{a^*_t}(t)\right)\right]
\end{equation*}
\label{lemma:russo}
\end{lemma}

Let us begin from Lemma~\ref{lemma:russo}. By the tower rule, we can condition over event $E$:
\begin{multline*}
    \text{BayesRegret}(T, \pi^{TS}) =  \mathbb E\left[\sum_{t=1}^T\left(U_{a_t}(t) - \mu_{a_t}\right) + \left(\mu_{a^*_t} - U_{a^*_t}(t)\right) \Bigg| E \right]\mathbb P(E) +\\ \mathbb E\left[\sum_{t=1}^T\left(U_{a_t}(t) - \mu_{a_t}\right) + \left(\mu_{a^*_t} - U_{a^*_t}(t)\right) \Bigg| E^c \right]\mathbb P(E^c)
\end{multline*}
Since $\mathbb P(E) \leq 1$,
\begin{multline*}
    \text{BayesRegret}(T, \pi^{TS}) \leq  \mathbb E\left[\sum_{t=1}^T\left(U_{a_t}(t) - \mu_{a_t}\right) + \left(\mu_{a^*_t} - U_{a^*_t}(t)\right) \Bigg| E \right] +\\ \mathbb E\left[\sum_{t=1}^T\left(U_{a_t}(t) - \mu_{a_t}\right) + \left(\mu_{a^*_t} - U_{a^*_t}(t)\right) \Bigg| E^c \right]\mathbb P(E^c)
\end{multline*}
When $E^c$ holds, each term in the summation in the conditional expectation is bounded by $4M$ (Equation~\ref{eqn:ucb}). Therefore, 
\begin{align*}
    \text{BayesRegret}(T, \pi^{TS}) &\leq  \mathbb E\left[\sum_{t=1}^T\left(U_{a_t}(t) - \mu_{a_t}\right) + \left(\mu_{a^*_t} - U_{a^*_t}(t)\right) \Bigg| E \right] + 4MT\cdot\mathbb P(E^c) \\ 
    &\stackrel{(a)}{\leq}  \mathbb E\left[\sum_{t=1}^T\left(U_{a_t}(t) - \mu_{a_t}\right) + \left(\mu_{a^*_t} - U_{a^*_t}(t)\right) \Bigg| E \right] + 4KMT^2\delta
    \\
& \stackrel{(b)}{\leq} 2\mathbb E\left[\sum_{k=1}^K\sum_{t=1}^{T} \mathbbm 1\{A_t = k\}\left(\frac{2C(1+\epsilon, \alpha)}{\delta n_k(t)^\epsilon}\right)^{\frac{1}{1+\epsilon}}\right] + 4KMT^2\delta\\
& = 2\left(\frac{2C(1+\epsilon, \alpha)}{\delta}\right)^{\frac{1}{1+\epsilon}}\mathbb E\left[\sum_{k=1}^K\sum_{t=1}^{T} \mathbbm 1\{A_t = k\} \left(\frac{1}{n_k(t)^\epsilon}\right)^{\frac{1}{1+\epsilon}}\right] + 4KMT^2\delta\\
&\stackrel{(c)}{\leq} 2\left(\frac{2C(1+\epsilon, \alpha)}{\delta}\right)^{\frac{1}{1+\epsilon}}\mathbb E\left[\sum_{k=1}^{K} \int_{s=0}^{n_k(T)}\left(\frac{1}{s^{\epsilon}}\right)^{\frac{1}{1+\epsilon}}ds\right] + 4KMT^2\delta\\
&= 2(1+\epsilon)\left(\frac{2C(1+\epsilon, \alpha)}{\delta }\right)^{\frac{1}{1+\epsilon}}\mathbb E\left[\sum_{k=1}^{K} n_k(T)^{\frac{1}{1+\epsilon}}\right]+ 4KMT^2\delta\\
&\stackrel{(d)}{\leq} 4\left(\frac{2C(1+\epsilon, \alpha)}{\delta }\right)^{\frac{1}{1+\epsilon}}\mathbb E\left[K^{\frac{1}{1+\epsilon}}\left(\sum_{k=1}^{K} n_k(T)\right)^{\frac{1}{1+\epsilon}}\right]+4KMT^2\delta\\
&\stackrel{(e)}{\leq} 4\left(\frac{2C(1+\epsilon, \alpha)}{\delta }\right)^{\frac{1}{1+\epsilon}}\left(KT\right)^{\frac{1}{1+\epsilon}}+4KMT^2\delta.
\end{align*}
Here, $(a)$ follows from Lemma~\ref{lemma:stable_product_comp}, $(b)$ follows from event $E$: whenever $E$ holds, each term inside the summation is bounded by Equation~(\ref{eqn:event_e_basic}), $(b)$ follows from the upper bound of a finite discrete sum with a definite integral, $(c)$ follows from H\"older's Inquality of order $\frac{1}{1+\epsilon}$ and $(d)$ follows from $T = 1 + \sum_{k=1}^K n_k(T)$, and that $1+\epsilon < 2$. For a finite time analysis, we can now choose $\epsilon$ and $T$ such that it minimizes the RHS. This concludes the proof.
\end{proof}
\subsection*{Regret Analysis: Robust $\alpha$-Thompson Sampling}
Bounding the concentration of the robust mean estimator first requires a bound on the raw moments of symmetric $\alpha$-stable densities, provided in the next two results.
\begin{lemma}[Proposition 2.2 of~\cite{matsui2016fractional}]
For any random variable $X \sim S_\alpha(0, \sigma, 0), \epsilon \in (-\infty, \alpha - 1)$ and $\nu \in \mathbb R$, 
\begin{equation}
    \mathbb E[|X - \nu|^{1+\epsilon}] = \frac{\epsilon\cdot \sigma^{1+\epsilon}}{\sin\left(\frac{\pi\cdot\epsilon}{2}\right)\Gamma(1-\epsilon)}\left[\frac{\nu}{\sigma}\int_0^\infty u^{-(1+\epsilon)}e^{-u^\alpha}\sin\left(\frac{\nu u}{\sigma}\right)du+\alpha\int_0^\infty u^{\alpha-\epsilon-2}e^{-u^\alpha}\cos\left(\frac{\nu u}{\sigma}\right)du\right]
\end{equation}
\label{lemma:raw_moments}
\end{lemma}
\begin{lemma}
For any random variable $X \sim S_\alpha(0, \sigma, \mu), \epsilon \in (-\infty, \alpha -1), \mu \leq M $,
\begin{equation}
    \mathbb E[|X|^{1+\epsilon}] \leq \frac{\epsilon\cdot \sigma^{1+\epsilon}\left(M \cdot \Gamma(-\epsilon/\alpha)+ \sigma\alpha\Gamma(1 - \frac{\epsilon + 1}{\alpha})\right)}{\sigma\alpha\sin\left(\frac{\pi\cdot\epsilon}{2}\right)\Gamma(1-\epsilon)}
\end{equation}
\label{lemma:raw_moment_bound}
\end{lemma}
\begin{proof}
Let $X \sim S_\alpha(0, \sigma, \mu)$. Then $X-\mu \sim S_\alpha(0, \sigma, 0)$. Applying Lemma~\ref{lemma:raw_moments} to $X-\mu$ with $\nu = -\mu, \epsilon \in (-\infty, \alpha -1)$, we have
\begin{align*}
    \mathbb E[|X|^{1+\epsilon}] &= \mathbb E[|X - \mu - (-\mu)|^{1+\epsilon}] \\
    &=  \frac{\epsilon\cdot \sigma^{1+\epsilon}}{\sin\left(\frac{\pi\cdot\epsilon}{2}\right)\Gamma(1-\epsilon)}\left[\frac{-\mu}{\sigma}\int_0^\infty u^{-(1+\epsilon)}e^{-u^\alpha}\sin\left(\frac{-\mu u}{\sigma}\right)du+\alpha\int_0^\infty u^{\alpha-\epsilon-2}e^{-u^\alpha}\cos\left(\frac{\mu u}{\sigma}\right)du\right] \\
    &\stackrel{(a)}{=}\frac{\epsilon\cdot \sigma^{1+\epsilon}}{\sin\left(\frac{\pi\cdot\epsilon}{2}\right)\Gamma(1-\epsilon)}\left[\frac{\mu}{\sigma}\int_0^\infty u^{-(1+\epsilon)}e^{-u^\alpha}\sin\left(\frac{\mu u}{\sigma}\right)du+\alpha\int_0^\infty u^{\alpha-\epsilon-2}e^{-u^\alpha}\cos\left(\frac{\mu u}{\sigma}\right)du\right] \\ 
    &\stackrel{(b)}{\leq} \frac{\epsilon\cdot \sigma^{1+\epsilon}}{\sin\left(\frac{\pi\cdot\epsilon}{2}\right)\Gamma(1-\epsilon)}\left[\frac{\mu}{\sigma}\int_0^\infty u^{-(1+\epsilon)}e^{-u^\alpha}du+\alpha\int_0^\infty u^{\alpha-\epsilon-2}e^{-u^\alpha}du\right] \\
    &\stackrel{(c)}{=}\frac{\epsilon\cdot \sigma^{1+\epsilon}}{\sin\left(\frac{\pi\cdot\epsilon}{2}\right)\Gamma(1-\epsilon)}\left[\frac{\mu}{\sigma\alpha}\int_0^\infty t^{-1-\epsilon/\alpha}e^{-t}dt+\int_0^\infty t^{-\frac{1+\epsilon}{\alpha}}e^{-t}dt\right]\\
    &\stackrel{(d)}{\leq}\frac{\epsilon\cdot \sigma^{1+\epsilon}\left(M\cdot \Gamma(-\epsilon/\alpha)+ \sigma\alpha\Gamma(1 - \frac{\epsilon + 1}{\alpha})\right)}{\sigma\alpha\sin\left(\frac{\pi\cdot\epsilon}{2}\right)\Gamma(1-\epsilon)}
\end{align*}
Here, $(a)$ follows from $\sin(-x) = -\sin(x)$, $(b)$ follows from $\sin(x) \leq 1, \cos(x) \leq 1\ \forall x$, $(c)$ follows by the substitution $t = u^\alpha$, and $(d)$ follows from $\mu \leq M$.
\end{proof}
Before stating the proof, we first describe the concentration results for the raw moments of the truncated mean estimator.
\begin{lemma}[Lemma 1 of~\cite{bubeck2013bandits}]
Let $\delta \in (0, 1), \epsilon \in (0, 1], u > 0$. Consider the truncated empirical mean $\hat\mu_T$ defined as,
\begin{equation}
    \hat\mu_T = \frac{1}{n}\sum_{t=1}^n X_t \mathbbm{1}\left\{|X_t| \leq \left(\frac{ut}{\log(\delta^{-1})}\right)^{\frac{1}{1+\epsilon}}\right\}
\end{equation}
If $\mathbb E[|X|^{1+\epsilon}] < u, \mathbb E[X] = \mu$, then with probability $1-\delta$,
\begin{equation}
    \hat\mu_t \leq \mu + 4u^{\frac{1}{1+\epsilon}}\left(\frac{\log(\delta^{-1})}{n}\right)^{\frac{\epsilon}{1+\epsilon}}
\end{equation}
\label{lemma:robust_mean_bound}
\end{lemma}
\begin{lemma}
For any $\delta \in (0, 1)$, arm $k \in [K]$, time $t > 0$, maximum possible mean $M > 0$, and $\epsilon \in (0, \alpha - 1)$, we have, with probability at least $1-\delta$,
\begin{equation*}
    |\hat r^*_k(t) - \mu_k| \leq 4\sigma\left(\frac{\epsilon\left(M \Gamma(-\epsilon/\alpha)+ \sigma\alpha\Gamma(1 - \frac{\epsilon + 1}{\alpha})\right)}{\sigma\alpha\sin\left(\frac{\pi\cdot\epsilon}{2}\right)\Gamma(1-\epsilon)}\right)^{\frac{1}{1+\epsilon}}\left(\frac{\log(\delta^{-1})}{n}\right)^{\frac{\epsilon}{1+\epsilon}}
\end{equation*}
\label{lemma:concentration_log}
\end{lemma}
\begin{proof}
For any arm $k$, we know that the rewards follow $S_\alpha(0, \sigma, \mu_k)$. Applying Lemma~\ref{lemma:raw_moment_bound} to the rewards distribution for arm $k$, we obtain $u = \frac{\epsilon\cdot \sigma^{1+\epsilon}\left(M\cdot \Gamma(-\epsilon/\alpha)+ \sigma\alpha\Gamma(1 - \frac{\epsilon + 1}{\alpha})\right)}{\sigma\alpha\sin\left(\frac{\pi\cdot\epsilon}{2}\right)\Gamma(1-\epsilon)}$. Using this value of $u$ for both tails of $\hat r^*_k(t)$ in Lemma~\ref{lemma:robust_mean_bound} with probability $\delta/2$, we obtain the result.
\end{proof}
Using the above concentration result, we can now derive a regret bound on the robust $\alpha$-Thompson Sampling algorithm. For convenience, let $H(\epsilon, \alpha, \sigma) = \left(\frac{\epsilon\left(M \cdot \Gamma(-\epsilon/\alpha)+ \sigma\alpha\Gamma(1 - \frac{\epsilon + 1}{\alpha})\right)}{\sigma\alpha\sin\left(\frac{\pi\cdot\epsilon}{2}\right)\Gamma(1-\epsilon)}\right)$. First, let us state the form of the robust mean estimator.
\begin{definition}
Let $ \delta \in (0, 1), \epsilon \in (0, \alpha-1), k \in [K]$ and time horizon be $T$. The truncated mean estimator $\hat r_k^*(t)$ is given by:
\begin{align}
    \hat r_k^*(t) = \frac{1}{n_k(t)}\sum_{i=1}^{n_k(t)}r_k^{(i)}\mathbbm{1}\left\{|r_k^{(i)}| \leq \left(\frac{H(\epsilon, \alpha, \sigma)\cdot i}{2\log(T)}\right)^{\frac{1}{1+\epsilon}} \right\}
    \label{eqn:truncated}
\end{align}
\end{definition}
\begin{theorem}[Robust Regret Bound]
Let $K> 1, \alpha \in (1, 2), \sigma \in \mathbb R^+, \mu_{k: k \in [K]} \in \mathbb [-M, M]$. For a $K$-armed bandit with rewards for each arm $k$ drawn from $S_\alpha(\beta, \sigma, \mu_k)$, we have, for $\epsilon$ chosen a priori such that $\epsilon \rightarrow (\alpha - 1)^-$ and truncated estimator given in Equation~(\ref{eqn:truncated}), 
\begin{equation*}
    \text{BayesRegret}(T, \pi^{RTS}) = O\left((KT)^{\frac{1}{1+\epsilon}}(\log T)^{\frac{\epsilon}{1+\epsilon}}\right)
\end{equation*}
\label{theorem:regret_log}
\end{theorem}
\begin{proof}
Consider a $K$-armed bandit with rewards for arm $k$ drawn from $S_\alpha(0, \sigma, \mu_k)$. Let $n_{k}(t)$ denote the number of times arm $k$ has been pulled until time $t$. Then $t-1 = \sum_{k=1}^K n_k(t)$. Let us denote the empirical average reward for arm $k$ up to (and including) time $t-1$ as $\hat{r}_{k}(t)$, and denote the arm pulled at any time $t$ as $a_t$, and the optimal arm as $a^*_t$. We then set an upper  confidence bound for arm $k$ at any time $t$ as 
\begin{equation}
    U_k(t) = \text{clip}_{[-M, M]}\left[\hat{r}_{k}(t) + 4\sigma H(\epsilon, \alpha, \sigma)^{\frac{1}{1+\epsilon}}\left(\frac{\log(2/\delta)}{n_k(t)}\right)^{\frac{\epsilon}{1+\epsilon}}\right] 
    \label{eqn:ucb}
\end{equation} 
for $0 < \epsilon < 1, M > 0$. Let $E$ be the event when for all $k \in [K]$ arms, over all iterations $t \in [T]$, we have:
\begin{equation}
    |\hat{r}_k(t) - \mu_k | \leq  4\sigma H(\epsilon, \alpha, \sigma)^{\frac{1}{1+\epsilon}}\left(\frac{\log(2/\delta)}{n_k(t)}\right)^{\frac{\epsilon}{1+\epsilon}}.
    \label{eqn:event_e}
\end{equation}
\begin{lemma}
For the setup described above, we have, for event $E$ and $\delta \in (0, 1)$,
\begin{equation*}
    \mathbb P(E^c) \leq KT\delta.
\end{equation*}
\label{lemma:product_log}
\end{lemma}
\begin{proof}
The event $E^c$ holds whenever the bound is violated for at least one arm $k$ at one instance $t$. Therefore,
\begin{align*}
    \mathbb P(E^c) &\leq \mathbb P\left( \bigcup_{\substack{k=1\\t=1}}^{K, T} \left\{|\hat{r}_k(t-1) - \mu_k| > 4\sigma H(\epsilon, \alpha, \sigma)^{\frac{1}{1+\epsilon}}\left(\frac{\log(2/\delta)}{n_k(t)}\right)^{\frac{\epsilon}{1+\epsilon}}\right\}\right)\\
    &\stackrel{(a)}{\leq} \sum_{\substack{k=1\\t=1}}^{K, T} \mathbb P\left(|\hat{r}_k(t-1) - \mu_k| > 4\sigma H(\epsilon, \alpha, \sigma)^{\frac{1}{1+\epsilon}}\left(\frac{\log(2/\delta)}{n_k(t)}\right)^{\frac{\epsilon}{1+\epsilon}}\right) \\
    &\stackrel{(b)}{\leq} KT\delta.
\end{align*}
Where $(a)$ is an application of the union bound, and $(b)$ is obtained using Lemma~\ref{lemma:concentration_log}.
\end{proof}
We now prove Theorem~\ref{theorem:regret_log} in a similar manner as Theorem~\ref{theorem:regret}, via Lemma~\ref{lemma:russo}. By the tower rule, we can condition over event $E$:
\begin{multline}
    \text{BayesRegret}(T, \pi^{RTS}) =  \mathbb E\left[\sum_{t=1}^T\left(U_{a_t}(t) - \mu_{a_t}\right) + \left(\mu_{a^*_t} - U_{a^*_t}(t)\right) \Bigg| E \right]\mathbb P(E) +\\ \mathbb E\left[\sum_{t=1}^T\left(U_{a_t}(t) - \mu_{a_t}\right) + \left(\mu_{a^*_t} - U_{a^*_t}(t)\right) \Bigg| E^c \right]\mathbb P(E^c)
\end{multline}
Since $\mathbb P(E) \leq 1$,
\begin{multline}
    \text{BayesRegret}(T, \pi^{RTS}) \leq  \mathbb E\left[\sum_{t=1}^T\left(U_{a_t}(t) - \mu_{a_t}\right) + \left(\mu_{a^*_t} - U_{a^*_t}(t)\right) \Bigg| E \right] +\\ \mathbb E\left[\sum_{t=1}^T\left(U_{a_t}(t) - \mu_{a_t}\right) + \left(\mu_{a^*_t} - U_{a^*_t}(t)\right) \Bigg| E^c \right]\mathbb P(E^c)
\end{multline}
When $E^c$ holds, each term in the summation in the conditional expectation is bounded by $4M$ (Equation~\ref{eqn:ucb}). Therefore, 
\begin{align*}
\text{BayesRegret}(T, \pi^{RTS}) &\leq   4MT\cdot\mathbb P(E^c) + \mathbb E\left[\sum_{t=1}^T\left(U_{a_t}(t) - \mu_{a_t}\right) + \left(\mu_{a^*_t} - U_{a^*_t}(t)\right) \Bigg| E \right] \\ 
&\stackrel{(a)}{\leq}   4KMT^2\delta +\mathbb E\left[\sum_{t=1}^T\left(U_{a_t}(t) - \mu_{a_t}\right) + \left(\mu_{a^*_t} - U_{a^*_t}(t)\right) \Bigg| E \right] \\
& \stackrel{(b)}{\leq}  4KMT^2\delta +8\mathbb E\left[\sum_{k=1}^K\sum_{t=1}^{T} \mathbbm 1\{A_t = k\} \sigma H(\epsilon, \alpha, \sigma)^{\frac{1}{1+\epsilon}}\left(\frac{\log(2/\delta)}{n_k(t)}\right)^{\frac{\epsilon}{1+\epsilon}}\right]\\
& =  4KMT^2\delta + 8\sigma H(\epsilon, \alpha, \sigma)^{\frac{1}{1+\epsilon}}\log(2/\delta)^{\frac{\epsilon}{1+\epsilon}}\mathbb E\left[\sum_{k=1}^K\sum_{t=1}^{T} \mathbbm 1\{A_t = k\} \left(\frac{1}{n_k(t)^\epsilon}\right)^{\frac{1}{1+\epsilon}}\right]\\
&\stackrel{(c)}{\leq} 4KMT^2\delta + 8\sigma H(\epsilon, \alpha, \sigma)^{\frac{1}{1+\epsilon}}\log(2/\delta)^{\frac{\epsilon}{1+\epsilon}}\mathbb E\left[\sum_{k=1}^{K} \int_{s=0}^{n_k(T)}\left(\frac{1}{s^{\epsilon}}\right)^{\frac{1}{1+\epsilon}}ds\right]\\
&=  4KMT^2\delta + 8(1+\epsilon)\sigma H(\epsilon, \alpha, \sigma)^{\frac{1}{1+\epsilon}}\log(2/\delta)^{\frac{\epsilon}{1+\epsilon}}\mathbb E\left[\sum_{k=1}^{K} n_k(T)^{\frac{1}{1+\epsilon}}\right]\\
&\stackrel{(d)}{\leq}  4KMT^2\delta + 16\sigma H(\epsilon, \alpha, \sigma)^{\frac{1}{1+\epsilon}}\log(2/\delta)^{\frac{\epsilon}{1+\epsilon}}\mathbb E\left[K^{\frac{1}{1+\epsilon}}\left(\sum_{k=1}^{K} n_k(T)\right)^{\frac{1}{1+\epsilon}}\right]\\
&\stackrel{(e)}{\leq} 4KMT^2\delta + 16\sigma H(\epsilon, \alpha, \sigma)^{\frac{1}{1+\epsilon}}\log(2/\delta)^{\frac{\epsilon}{1+\epsilon}}\left(KT\right)^{\frac{1}{1+\epsilon}}\\
&\stackrel{(f)}{=} 4KM + 16\sigma H(\epsilon, \alpha, \sigma)^{\frac{1}{1+\epsilon}}\left(\log(2)^{\frac{\epsilon}{1+\epsilon}}\right) + 32\sigma H(\epsilon, \alpha, \sigma)^{\frac{1}{1+\epsilon}}\log (T)^{\frac{\epsilon}{1+\epsilon}}\left(KT\right)^{\frac{1}{1+\epsilon}}.
\end{align*}
Here, $(a)$ follows from Lemma~\ref{lemma:product_log}, $(b)$ follows from event $E$: whenever $E$ holds, each term inside the summation is bounded by Equation~(\ref{eqn:event_e}), $(c)$ follows from the upper bound of a finite discrete sum with a definite integral, $(d)$ follows from H\"older's Inquality of order $\frac{1}{1+\epsilon}$ and $(e)$ follows from $T = 1 + \sum_{k=1}^K n_k(T)$, and that $1+\epsilon < 2$, and $(f)$ is obtained by setting $\delta = 1/T^2$.  Asymptotically,for $\epsilon$ chosen \textit{a priori} close to $\alpha-1$, 
\begin{equation}
    \text{BayesRegret}(T, \pi^{RTS}) = O((KT)^{\frac{1}{1+\epsilon}}(\log T)^{\frac{\epsilon}{1+\epsilon}}).
\end{equation} 
\end{proof}
We see that this bound is tight: when $\alpha = 2$, $\epsilon$ can be set to $1$, and we see that this matches the upper bound of $O(\sqrt{KT\log T})$\cite{russo2014learning} for the Gaussian case, which has also been shown to be optimal~\cite{agrawal2013further}. Hence, this alternative Thompson Sampling can be beneficial, with the performance improvements increasing with $\alpha$.
\bibliographystyle{alpha}
\bibliography{refs}
\end{document}